\documentclass[american,a4paper,runningheads,envcountsect]{llncs}

\usepackage{babel}
\usepackage[utf8]{inputenc}
\usepackage[T1]{fontenc}

\usepackage{enumitem}
\usepackage{amsmath,amsthm,amssymb}
\usepackage{mathtools}
\usepackage{mathrsfs}
\usepackage{fca}
\usepackage{dsfont} 

\usepackage[colorlinks,citecolor=blue,urlcolor=black,hidelinks,linktocpage]{hyperref}
\usepackage[all]{hypcap}
\usepackage{cleveref}
\let\cref\Cref

\newtheorem{thm}{Theorem}[section]

\newtheorem{lem}[thm]{Lemma}
\newtheorem{prop}[thm]{Proposition}

\theoremstyle{definition}

\newtheorem{defn}[thm]{Definition}

\theoremstyle{remark}

\renewcommand{\epsilon}{\varepsilon}
\renewcommand{\phi}{\varphi}

\newcommand{\K}{\mathbb{K}}
\newcommand{\B}[1]{\mathfrak{B}(#1)}
\newcommand{\BB}{\mathfrak{B}}
\usepackage[shrink=25,babel,kerning=true]{microtype}
\usepackage{csquotes}
\usepackage{booktabs}
\usepackage[subtle]{savetrees}
\usepackage{todonotes}
\presetkeys{todonotes}{color=blue!5}{}

\usepackage[backend=bibtex,style=numeric-comp,doi=false,isbn=false,%
url=false]{biblatex}
\newcommand\blfootnote[1]{%
  \begingroup
  \renewcommand\thefootnote{}\footnote{#1}%
  \addtocounter{footnote}{-1}%
  \endgroup
}

\begin{document}

\title{Relevant Attributes in Formal Contexts}

\author{Tom Hanika\inst{1,2} \and Maren Koyda\inst{1,2} \and Gerd Stumme\inst{1,2}}

\date{\today}

\institute{%
  Knowledge \& Data Engineering Group,
  University of Kassel, Germany\\[0.5ex]
  \and
  Interdisciplinary Research Center for Information System Design\\
  University of Kassel, Germany\\[0.5ex]
  \email{tom.hanika@cs.uni-kassel.de, koyda@cs.uni-kassel.de,
    stumme@cs.uni-kassel.de}
}
\maketitle

\blfootnote{Authors are given in alphabetical order.
  No priority in authorship is implied.}

\begin{abstract}
  Computing conceptual structures, like formal concept lattices, is in the age
  of massive data sets a challenging task. There are various approaches to deal
  with this, e.g., random sampling, parallelization, or attribute extraction. A
  so far not investigated method in the realm of formal concept analysis is
  attribute selection, as done in machine learning. Building up on this we
  introduce a method for attribute selection in formal contexts. To this end, we
  propose the notion of relevant attributes which enables us to define a
  relative relevance function, reflecting both the order structure of the
  concept lattice as well as distribution of objects on it. Finally, we overcome
  computational challenges for computing the relative relevance through an
  approximation approach based on information entropy.
\end{abstract}


\section{Introduction}
\label{sec:introduction}

The increasing number of features (attributes) in data sets poses a challenge
for many procedures in the realm of knowledge discovery. In particular, methods
employed in \emph{formal concept analysis} (FCA) become more infeasible for
large numbers of attributes. Of peculiar interest there is the construction,
visualization and interpretation of \emph{formal concept lattices}, an
algebraic structure usually represented through line or order diagrams.

The data structure used in FCA is a \emph{formal context}, roughly a data table
where every row represents an object associated to attributes described through
columns. Contemporary such data sets consist of thousands of rows and columns.
Since the computation of all formal concepts is at best possible with polynomial
delay~\cite{JOHNSON1988119}, thus sensitive to the output size, it is almost
unattainable to be computed even for moderately large sized data sets. The
problem for the computation of valid (attribute) implications is even more
serious since enumerating them is not possible with polynomial
delay~\cite{DISTEL2011450} (in lexicographic order) and only few algorithms are
known to compute them~\cite{DBLP:journals/amai/ObiedkovD07}. Furthermore, in
many applications storage space is limited, e.g., mobile computing or
decentralized embedded knowledge systems.

To overcome both the computational infeasibility as well as the storage
limitation one is required to select a sub-context resembling the original data
set most accurately. This can be done by selecting attributes, objects, or
both. In this work we will focus on the identification of relevant
attributes. This is, due to the duality of formal contexts, similar to the
problem of selecting relevant objects. There are several comparable works, e.g.,
~\cite{Kumar}, where the author investigated the applicability of random
projections. For supervised machine learning tasks there are even more
sophisticated methods utilizing filter approaches, which are based on the
distribution of labels~\cite{yu2003feature}. Works more related to FCA resort,
e.g., to concept sampling~\cite{sampling} and concept
selection~\cite{conceptsselection2010}. Both approaches, however, either need to
compute the whole (possibly large) concept lattice or sample from it.

In this work we overcome this limitation and present a feasible approach for
selecting relevant attributes from a formal context using information
entropy. To this end we introduce the notion for \emph{attribute relevance} to
the realm of FCA, based on a seminal work by Blum and
Langley~\cite{BLUM1997}. In there the authors address a comprehensible theory
for selecting most relevant features in supervised machine learning
settings. Building up on this we formalize a \emph{relative relevance} measure in
formal contexts in order to identify the most relevant attributes. However, this
measure is still prone to the limitation for computing the concept
lattice. Finally, we tackle this disadvantage by approximating the relative
relevance measure through an information entropy approach. Choosing attributes
based on this approximation leads to significantly more relevant selections than
random sampling does, which we demonstrate in an empirical experiment.

As for the structure of this paper, in Section 2 we  give a short overview
over the previous works applied to relevant attribute selection. Subsequently we
recall some basic notions from FCA followed by our definitions of relevance and
relative relevance of attribute selections and its approximations. In Section
4 we illustrate and evaluate our notions through experiments showing the
approximations are significantly superior to random sampling. We conclude our work
and give an outlook in Section 5.

\section{Related Work}
\label{sec:related-work}
In the field of supervised machine learning there are numerous approaches for
feature set selection. The authors from~\cite{JOHN1994121} introduced a
beneficial categorization for those in two categories: wrapper models and
filters. The wrapper models evaluate feature subsets using the underlying
learning algorithm. This allows to respond to redundant or correlated
features. However, these models demand many computations and are prone to
reproduce the procedural bias of the underlying learning algorithm. A
representative for this model type is the class of selective Bayesian
classifiers in Langley and Sage~\cite{Langley:1994}. In there the authors
extended the naive Bayesian classifier by considering subsets of given feature
sets only for predictions. The other category from~\cite{JOHN1994121} is filter
models. Those work independently from the underlying learning algorithm. Instead
these methods make use of general characteristics like the attribute
distribution with respect to the labels in order to weight an attribute's
importance. Hence, they are more efficient but are likely to select redundant or
futile features with respect to an underlying machine learning procedure. A
well-known method representing this class is RELIEF~\cite{Relief}, which denotes
the relevance of all features referring to the class label using a statistical
method. An entropy based approach of a filter model was introduced by Koller et
al.~\cite{Koller}. There the authors introduced selecting features based on the
Kullback-Leibler-distance. All these methods incorporate an underlying notion of
attribute relevance. This notion was captured and formalized in the seminal work
by Blum and Langley in~\cite{BLUM1997}, on which we will base the notion for
relevant attributes in formal contexts.

There are some approaches in FCA to face the attribute selection
problem. In~\cite{Kumar} a procedure based on random projection was
developed. Less related are methods employed after computing the formal concept
lattice, e.g., concept sampling~\cite{sampling} and concept
selection~\cite{conceptsselection2010}. Those could be compared to methods
from~\cite{Langley:1994}, as they first compute the concept lattice.  More
related works originate from granular computing with FCA. A basic idea there is
to find information granules based on entropy. To this end the
authors of~\cite{LOIA20181} introduced an (object) entropy function for formal
contexts, which we will utilize in this work as well. Their approach used the
principles of granulation as in~\cite{ZADEH1997111}, which is based on merging
attributes to reduce the data set. Since our focus is on selecting attributes,
we turn away from this notion in general.

\section{Relevant Attributes}
\label{sec:inter-attr}
Before we start with our definition for relevant attributes of a formal context,
we want to recall some basic notions from formal concept analysis. For a
thorough introduction we refer the reader to~\cite{fca-book}.  A formal context
is triple $\mathbb{K}\coloneqq(G,M,I)$, where $G$ and $M$ are finite sets called
\emph{object set} and \emph{attribute set}, respectively. Those are connected
through a binary relation $I\subseteq G\times M$, called \emph{incidence}. If
$(g,m)\in I$ for an object $g\in G$ and an attribute $m\in M$, we write $gIm$
and say ``object $g$ has attribute $m$''.  On the power set of the objects
(power set of the attributes) we introduce two operators
$\cdot'\colon\mathcal{P}(G)\to\mathcal{P}(M)$, where $A\mapsto A'\coloneqq
\{m\in M\mid \forall g\in A\colon (g,m)\in I\}$ and
$\cdot'\colon\mathcal{P}(M)\to\mathcal{P}(G)$, where $B\mapsto B'\coloneqq\{g\in
G\mid \forall m\in B\colon (g,m)\in I\}$.  A pair $(A,B)$ with $A\subseteq G$
and $B\subseteq M$ is called \emph{formal concept} of the context $(G,M,I)$ iff
$A'=B$ and $B'=A$. For a formal concept $c=(A,B)$ the set $A$ is called the
\emph{extent} ($\extent{c}$) and $B$ the \emph{intent} ($\intent{c}$). For two
concepts $(A_1,B_1)$ and $(A_2,B_2)$ there is a natural partial order given
through $(A_1,B_1)\leq(A_2,B_2)$ iff $A_1\subseteq A_2$.  The set of all formal
concepts of some formal context $\mathbb{K}$, denoted by
$\mathfrak{B}(\mathbb{K})$, together with the just introduced partial order
constitutes the \emph{formal concept lattice}
$\underline{\mathfrak{B}}(\mathbb{K})\coloneqq(\mathfrak{B}(\mathbb{K}),\leq)$.

A severe computational problem in FCA is to compute the set of all formal
concepts, which resembles the CLIQUE problem~\cite{JOHNSON1988119}. Furthermore,
the number of formal concepts in a proper sized real-world data set tends to be
very large, e.g., 238710 in the (small) mushroom data set,
see~\cref{sec:data-set-description}. Hence, concept lattices for contemporary
sized data sets are hard to grasp and hard to cope with through consecutive
measures and metrics. Thus, a need for selecting sub-contexts from data sets or
sub-lattices is self-evident. This selection can be conducted in the formal
context as well as in the concept lattice. However, the computational feasible
choice is to do this in the formal context. Considering a induced sub-context
can be done in general in three different ways: One may consider only a subset
$\hat G\subseteq G$, a subset $\hat M\subseteq M$, or a combination of
those. Our goal for the rest of this work is to identify relevant attributes in
a formal context. The notion for (attribute) relevance shall cover two aspects:
the lattice structure and the distribution of objects on it. The task at hand is
to choose the most relevant attributes which do both reflect a large part of the
lattice structure as well as the distribution of the objects on the
concepts. For this we will introduce in the next section a notion for
\emph{relevant attributes} in a formal context. Due to the duality in FCA this
can easily be translated to object relevance.

\subsection{Choosing Attributes}
\label{sec:choosing-attributes}
There is plenitude of conceptions for describing the relevance of an attribute
in a data set. Apparently, the relevance should depend on the particular machine
learning or knowledge discovery procedure. One very influential work in this
direction was done by Blum and Langley in~\cite{BLUM1997}, where the authors
defined the \emph{(weak/strong) relevance} of an attribute in the realm of
labeled data. In particular, for some data set of examples $D$, described using
features from some feature set $F$, where every $d\in D$ has the label
(distribution) $\ell(d)$, the authors stated: A feature $x\in F$ is
\emph{relevant} to a target concept-label if there exists a pair of examples
$a,b\in D$ such that $a$ and $b$ only differ in their assignment of $x$ and
$\ell(a)\neq \ell(b)$. They further expanded their notion calling some attribute
$x$ is \emph{weakly relevant} iff it is possible to remove a subset of the
features (from $a$ and $b$) such that $x$ becomes relevant.

Since in the realm of formal concept analysis data is commonly unlabeled we may
not directly adapt the above notion to formal contexts. However, we may motivate
the following approach with it. We cope with the lack of a label function in the
following way. First, we identify the data set $D$ with a formal context
$(G,M,I)$, where the elements of $G$ are the examples and $M$ are the features
describing the examples. Secondly, a formal concept lattice exhibits essentially
two almost independent properties, the order structure and the distribution of
objects (attributes) on it, cf.~\cref{ex:distro}. Thus, a \emph{conceptual label
  function} then shall reflect both the order structure as well as the
distribution of objects in this structure. To achieve this we propose the
following.
\begin{defn}[Extent Label Function]
  \label{def:conlab}
  Let $\mathbb{K}\coloneqq (G,M,I)$ be a formal context and its concept lattice
  $\underline{\mathfrak{B}}(\mathbb{K})$. The map
  $\ell_{\mathbb{K}}:G\to\mathbb{N}, g\mapsto |\{c\in\mathfrak{B}(\mathbb{K})\mid
  g\in\extent{c}\}|$ is called \emph{extent label function}.
\end{defn}

One may define an \emph{intent label function} analogously.  Utilizing the just
introduced label function we may now define the notion of relevant attributes
in formal contexts.

\begin{defn}[Relevance]
  Let $\mathbb{K}\coloneqq(G,M,I)$ be a formal context. We say an attribute $m\in M$ is
  \textit{relevant to $g\in G$} if and only if
  $\ell_{\mathbb{K}_{\{m\}}}(g)<\ell_{\mathbb{K}}(g)$,
  where $\mathbb{K}_{\{m\}}\coloneqq (G,M\setminus \{m\},I\cap G\times
  (M\setminus\{m\}))$. Furthermore, $m$ is \emph{relevant to a subset}
  $A\subseteq G$ iff there is a $g\in A$ such that $m$ is relevant to $g$. And,
  we say $m$ is \emph{relevant to the context $\mathbb{K}$} iff $m$ is relevant
  to $G$.
\end{defn}

\begin{example} 
  \label{ex:distro}
  Figure \ref{runningexp} (right) shows a formal context and its concept
  lattice. The objects from there are abbreviated by their first letter in the
  following. The extent label function of the objects can easily be read from
  the lattice and is given by $\ell_{\mathbb{K}}(B)=2,\ \ell_{\mathbb{K}}(F)=4,\
  \ell_{\mathbb{K}}(D)=2,\ \ell_{\mathbb{K}}(S)=3$.  Additionally, one can
  deduct the relevant attributes. E.g., for attribute $b$ the equality
  $\ell_{\K_{\{b\}}}(D)=\ell_{\mathbb{K}}(D)$ holds. In contrast
  $\ell_{\K_{\{b\}}}(S)<\ell_{\mathbb{K}}(S)$,
  cf.~\cref{exp:sublattices}. Hence, attribute $b$ is not relevant to ``Dog''
  but relevant to ``Spike-weed''. Thus, $b$ is relevant to $\K$.
\end{example}

\begin{figure}[t]
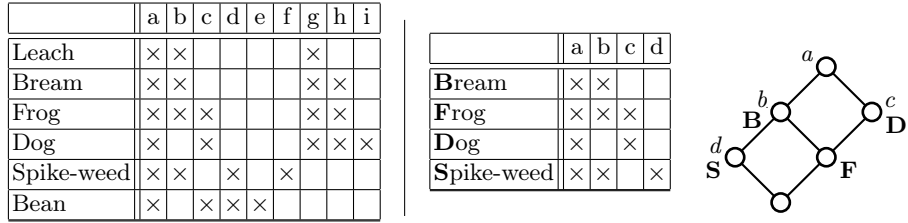

  \label{runningexp}
  \begin{minipage}{.44\textwidth}
    \centering
    \begin{cxt}%
      \cxtName{}%
      \att{a}%
      \att{b}%
      \att{c}%
      \att{d}%
      \att{e}%
      \att{f}%
      \att{g}%
      \att{h}%
      \att{i}%
      \obj{xx....x..}{Leach}%
      \obj{xx....xx.}{Bream}%
      \obj{xxx...xx.}{Frog}%
      \obj{x.x...xxx}{Dog}%
      \obj{xx.x.x...}{Spike-weed}%
      \obj{x.xxx....}{Bean}%
    \end{cxt}
  \end{minipage}
  \rule[-4em]{0.1ex}{8em}\ 
  \begin{minipage}{.30\textwidth}
    \centering
    \begin{cxt}%
      \cxtName{}%
      \att{a}%
      \att{b}%
      \att{c}%
      \att{d}%
      \obj{xx..}{\textbf{B}ream}%
      \obj{xxx.}{\textbf{F}rog}%
      \obj{x.x.}{\textbf{D}og}%
      \obj{xx.x}{\textbf{S}pike-weed}%
    \end{cxt} 		
  \end{minipage}
  \begin{minipage}{.24\textwidth}
    \begin{flushright}
      {\unitlength 0.6mm
        \begin{diagram}{40}{50}
          \Node{1}{0}{15}
          \Node{2}{10}{5}
          \Node{3}{20}{15}
          \Node{4}{10}{25}
          \Node{5}{30}{25}
          \Node{6}{20}{35}
          \Edge{1}{2}
          \Edge{1}{4}
          \Edge{2}{3}
          \Edge{3}{4}
          \Edge{3}{5}
          \Edge{6}{4}
          \Edge{6}{5}
          \leftAttbox{6}{3}{1}{a} 
          \rightAttbox{5}{3}{1}{c}
          \leftAttbox{1}{3}{1}{d}
          \leftAttbox{4}{3}{1}{b} 
          \rightObjbox{5}{3}{1}{\textbf{D}}
          \leftObjbox{1}{3}{1}{\textbf{S}}
          \rightObjbox{3}{3}{1}{\textbf{F}}
          \leftObjbox{4}{4}{0}{\textbf{B}}
        \end{diagram}}
    \end{flushright}
  \end{minipage}
  \caption{Sub-contexts of "Living Beings and Water"~\cite{fca-book}. The attributes are: a: needs water to live, b: lives in water, c: lives on land, d: needs chlorophyll to produce food, e: two seed leaves, f: one seed leaf, g: can move around, h: has limbs, i: suckles its offspring}
\end{figure}

There are two structural approaches in FCA to identify admissible attributes,
namely \emph{attribute clarifying} and \emph{reducibility}. Those are based
purely on the lattice structure. A formal context $\mathbb{K}\coloneqq(G,M,I)$
is called \emph{attribute clarified} iff for all attributes $m,n\in M$ with
$m'=n'$ follows that $m=n$. If there is furthermore no $m\in M$ and $X\subseteq
M$ with $m'=X'$ the context is called \emph{attribute reduced}. Analogously, the
terms \emph{object clarified} and \emph{object reduced} can be determined. An
attribute and object clarified (reduced) context is simply called
\emph{clarified} (\emph{reduced}). The concept lattice of the clarified/reduced
context is isomorphic to the concept lattice of the original context. If one of
these properties does not hold for an attribute (or an object) the context can
be can clarified/reduced by eliminating all such attributes (objects).
Obviously, the notion for relevant attributes is related to reducibility.

\begin{lem}[Irreducible]
  \label{lem:irred}
  For $m\in M$ in $\mathbb{K}=(G,M,I)$ holds
  \[m \ \text{is relevant to}\ \mathbb{K} \iff m\ \text{is irreducible}.\]  
\end{lem}
\begin{proof}
  We first show ($\Rightarrow$).  We have to show that the following  inequality holds:
  $|\{c\in\B{\K}\mid g\in\extent{c}\}|\leq|\{c\in\B{\K_{\{m\}}}\mid g\in
  \extent{c}\}|$.  Since $g\in\extent{c}$ and for any $c\in\B{\K}$ exists a
  unique concept $\hat c\in\B{\K_{\{m\}}}$ with $\intent{\hat
    c}\cup\{m\}=\intent{c}$, cf.~\cite[pg 24]{fca-book}, we have that
  $g\in(\intent{\hat c}\cup\{m\})'\subseteq \intent{\hat c}'$.  For
  ($\Leftarrow$) we employ~\cite[Prop. 30]{fca-book}, i.e., there is a
  join preserving order embedding
  $(G,M\setminus{m},I\cap(G\times(M\setminus\{m\})))\to (G,M,I)$ with
  $(A,B)\mapsto (A,A')$. Hence, every extent in $\B{\K_{\{m\}}}$ is also an
  extent in $\B{\K}$ which implies for all $g\in G$ that
  $\ell_{\K_{\{m\}}}(g)<\ell_{\K}(g)$.
\end{proof}

The last lemma implies that no clarifiable attributes would be considered as
relevant, even if the removal of all attributes that have the same closure would
have a huge impact on the structure of the concept lattice. Therefore a
meaningful identification of relevant attributes restrains to the identification
of meaningful equivalence classes $[x]_{\K}\coloneqq\{y\in M\mid x'=y'\}$ for
all $y\in M$. Accordingly we consider in the following only clarified
contexts. Transferring the relevance of an attribute $m \in M $ to its
equivalence class is an easy task which can be executed if necessary.

So far we are only able to decide for the relevance of an attribute but not
discriminate attributes upon their relevancy to the concept lattice. To overcome
this limitation we introduce in the following a measure which is able to compare
the relevancy of two given attributes in a clarified formal context. We consider
the change in the object label distribution $\{(g,\ell_{\K}(g))\mid g\in G\}$
going from $\K$ to $\K_{\{m\}}$ as characteristic to the relevance of a relevant
attribute $m$. To examine this characteristic in more detail and to make it
graspable via an numeric value we propose the following inequality: $\sum_{g\in
  G}\ell_{\mathbb{K}_{\{m\}}}(g)<\sum_{g\in G}\ell_{\mathbb{K}}(g)$.  This
approach offers not only the possibility to verify the existence of a change in
the object label distribution but also to measure the extent of this change. We
may quantify this  via $\sum_{g\in
  G}\ell_{\mathbb{K}_{\{m\}}}(g)/\sum_{g\in G}\ell_{\mathbb{K}}(g)\eqqcolon
t(m)$ whence $t(m) < 1$ for all  attributes $m\in M$.

\begin{defn}[Relative Relevance]
	\label{def:relr}
	Let $\mathbb{K}=(G,M,I)$ be a clarified formal context. The attribute
        $m\in M$ is \textit{relative relevant} to $\mathbb{K}$ with
	\[r(m)\coloneqq1- \frac{\sum_{g\in G} |\{c\in\mathfrak{B}(\mathbb{K}_{\{m\}})\mid
		g\in\extent{c}\}|}{\sum_{g\in G}  |\{c\in\mathfrak{B}(\mathbb{K})\mid
		g\in\extent{c}\}|} = 1-t(m).\]
\end{defn} 

The values of $r(m)$ for an attribute are in $[0,1)$. We say $m\in M$ is
\emph{more relevant} to $\K$ than $n\in M$ iff $r(n)< r(m)$. Double counting
leads to the following proposition.

\begin{prop}
  \label{lem:sum}
  Let $\mathbb{K}=(G,M,I)$ be a formal context. For  all $m\in M$ holds 
  \[r(m)=1- \frac{\sum_{c\in\B{\K}_{\{m\}}} |ext(c)|}{\sum_{c\in\mathfrak{B}} |ext(c)|}\]
  with $\B{\K}_{\{m\}}= \{c\in\mathfrak{B}| (int(c)\setminus\{m\})'=ext(c)\}$. 
\end{prop}
\begin{figure}[t]
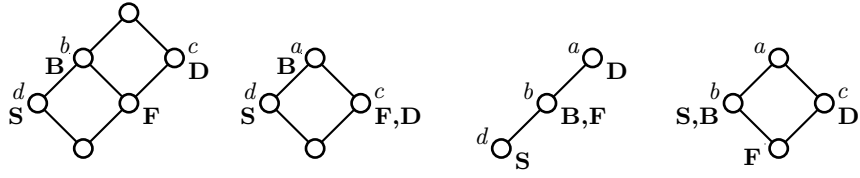

  \label{exp:sublattices}
  \begin{flushright}
    {\unitlength 0.6mm				
      \begin{diagram}{40}{50}
        \Node{1}{0}{15}
        \Node{2}{10}{5}
        \Node{3}{20}{15}
        \Node{4}{10}{25}
        \Node{5}{30}{25}
        \Node{6}{20}{35}
        \Edge{1}{2}
        \Edge{1}{4}
        \Edge{2}{3}
        \Edge{3}{4}
        \Edge{3}{5}
        \Edge{6}{4}
        \Edge{6}{5}
        \rightAttbox{5}{3}{1}{c}
        \leftAttbox{1}{3}{1}{d}
        \leftAttbox{4}{3}{1}{b} 
        \rightObjbox{5}{3}{1}{\textbf{D}}
        \leftObjbox{1}{3}{1}{\textbf{S}}
        \rightObjbox{3}{3}{1}{\textbf{F}}
        \leftObjbox{4}{4}{0}{\textbf{B}}
      \end{diagram}}\quad\quad
    {\unitlength 0.6mm
      \begin{diagram}{40}{50}
        \Node{1}{0}{15}
        \Node{2}{10}{5}
        \Node{3}{20}{15}
        \Node{4}{10}{25}
        \Edge{1}{2}
        \Edge{1}{4}
        \Edge{2}{3}
        \Edge{3}{4}
        \rightAttbox{3}{3}{1}{c}
        \leftAttbox{1}{3}{1}{d}
        \leftAttbox{4}{3}{1}{a} 
        \leftObjbox{1}{3}{1}{\textbf{S}}
        \rightObjbox{3}{3}{1}{\textbf{F,D}}
        \leftObjbox{4}{4}{0}{\textbf{B}}
      \end{diagram}}\quad\quad
    {\unitlength 0.6mm
      \begin{diagram}{40}{50}
        \Node{1}{0}{5}
        \Node{2}{10}{15}
        \Node{3}{20}{25}
        \Edge{1}{2}
        \Edge{3}{2}
        \leftAttbox{3}{3}{1}{a}
        \leftAttbox{1}{3}{1}{d}
        \leftAttbox{2}{3}{1}{b} 
        \rightObjbox{3}{3}{1}{\textbf{D}}
        \rightObjbox{1}{3}{1}{\textbf{S}}
        \rightObjbox{2}{3}{1}{\textbf{B,F}}
      \end{diagram}}\quad\quad
    {\unitlength 0.6mm
      \begin{diagram}{40}{50}
        \Node{1}{0}{15}
        \Node{2}{10}{5}
        \Node{3}{20}{15}
        \Node{4}{10}{25}
        \Edge{1}{2}
        \Edge{1}{4}
        \Edge{2}{3}
        \Edge{3}{4}
        \rightAttbox{3}{3}{1}{c}
        \leftAttbox{1}{3}{1}{b}
        \leftAttbox{4}{3}{1}{a} 
        \leftObjbox{1}{3}{1}{\textbf{S,B}}
        \rightObjbox{3}{3}{1}{\textbf{D}}
        \leftObjbox{2}{4}{0}{\textbf{F}}
      \end{diagram}}
    \end{flushright}
    \caption{Sub-lattices created through the removal of an attribute
      from~\cref{runningexp} (right). From left to right: removing a,b,c, or d.}
\end{figure}

This statement reveals an interesting property of the just defined relative
relevance. In fact, an attribute $m\in M$ is more relevant to an formal context
$\K$ if the join preserving sub-lattice, which one does obtain by removing $m$
from $\K$, does exhibit a smaller sum of all extent sizes. This will enable us
to find proper approximations to the relative relevance
in~\cref{sec:appr-relev}.

\begin{example}
  \label{exp:relrelevance}
  Excluding one attribute from the running example in~\cref{runningexp} (right)
  results in the sub-lattices in~\cref{exp:sublattices}.  The relative relevance
  of the attributes to the original context is given by $r(a)=0$, $r(b)=4/11$,
  $r(c)=3/11$, and $r(d)=1/11$.
\end{example}

By means of $r(\cdot)$ it is also possible to measure the relative relevance of
a set $N\subseteq M$. We simply lift~\ref{lem:sum} by $r(N)=1-
\sum_{c\in\B{\K}_{N}} |\extent{c}|\,/\,\sum_{c\in\B{\K}} \mid \extent{c}|$ with
$\B{\K}_{N}= \{c\in\B{K}\mid (\intent{c}\setminus\{N\})'=\extent{c}\}$.

\begin{lem}
	\label{lem:d-ungl}
	Let $\mathbb{K}=(G,M,I)$ be a formal context and $S,T\subseteq M$
        attribute sets. Then
        \begin{enumerate}[label=\roman*)]
        \item\label{lemdungl1} $S\subseteq T\implies r(S)\leq r(T)$, and
        \item\label{lemdungl2} $r(S\cup T)\leq r(T)+r(S)$. 
        \end{enumerate}
\end{lem}

\begin{proof}
  We prove~\ref{lemdungl1} by showing $\sum_{c\in\mathfrak{B}_{S}}
  |\extent{c}|>\sum_{c\in\mathfrak{B}_{T}} |\extent{c}|$.  Since
  $\forall c\in\B{\K}$ we have $(\intent{c}\setminus T)'\supseteq (\intent{c}\setminus
  S)'\supseteq\extent{c}$ we obtain $\B{\K}_{S}\supseteq \B{\K}_{T}$, as
  required.

  For ~\ref{lemdungl2} we will use the identity ($\star$):
  $\B{\K}_{S}\cap\B{\K}_{T}=\B{\K}_{S\cup T}$, which follows from
   $(\intent{c}\setminus S)'=\extent{c}\wedge (\intent{c}\setminus
  T)'=\extent{c}\iff (\intent{c}\setminus (S\cup T))'=\extent{c}$ for all
  $c\in\B{\K}$.  This equivalence is true since ($\Rightarrow$):
  \begin{align*}
    (\intent{c}\setminus (S\cup T))'&=
                                       ((\intent{c}\setminus S)\cap (\intent{c}\setminus T))'\\
                                     &=(\intent{c}\setminus
                                       S)'\cup(\intent{c}\setminus T)' = \extent{c}\cup\extent{c}=\extent{c}
  \end{align*}
  ($\Leftarrow$): From $(\intent{c}\setminus(S\cup T))'\supseteq
  (\intent{c}\setminus S)'$ and $(\intent{c}\setminus(S\cup T))'\supseteq
  (\intent{c}\setminus T)'$ we obtain with~\ref{lemdungl1} that
  $(\intent{c}\setminus S)'=\extent{c}$.  We now show~\ref{lemdungl2} by proving
  the inequality
  $\sum_{\B{\K}_{S}}|\extent{c}|+\sum_{\B{\K}_{T}}|\extent{c}|\leq
  \sum_{\B{\K}}|\extent{c}|+\sum_{\B{\K}_{S\cup T}}|\extent{c}|$. Using
  $\B{\K}_{S}\setminus\B{\K}_{S\cup T}\cup\B{\K}_{S\cup T}= \B{\K}_{S}$ where
  $\B{\K}_{S}\setminus\B{\K}_{S\cup T}\cap\B{\K}_{S\cup T}=\emptyset$ we find an
  equivalent equation employing ($\star$):
  \begin{align*}
    \sum_{\mathclap{\BB_{S}\setminus \BB_{S\cup T}}}|\extent{c}| +
    \sum_{\mathclap{\BB_{T}\setminus \BB_{S\cup T}}}|\extent{c}| +
    2\cdot \sum_{\mathclap{\BB_{S\cup T}}}|\extent{c}|
    \quad\leq\quad
    &\sum_{\mathclap{\BB_{S}\setminus \BB_{S\cup T}}}|\extent{c}| +
    \sum_{\mathclap{\BB_{T}\setminus \BB_{S\cup T}}}|\extent{c}| +\\
    &\sum_{\mathclap{\BB\setminus (\BB_{S}\cup\BB_{T})}}|\extent{c}| +
      2\cdot \sum_{\mathclap{\BB_{S\cup T}}}|\extent{c}|\\
    0\quad\leq\quad &\sum_{\mathclap{\BB\setminus (\BB_{S}\cup\BB_{T})}}|\extent{c}|
  \end{align*}
  where $\BB_{X}$ is short for $\B{\K}_{X}$.
\end{proof}

Equipped with the notion for relative relevance and some basic observations we
are ready to state the associated computational problem. We imagine that in
real-world applications attribute selection is a task to identify a set
$N\subseteq M$ of the most relevant attributes for a given cardinality
$n\in\mathbb{N}$, i.e., an element from $\{N\subseteq M\mid |N|=n\wedge r(N)\
\text{maximal}\}$. We call such a set $N$ a \emph{maximal relevant set}.

\begin{problem}[Relative Relevance Problem (RRP)]
  \label{rrp}
  Let $\mathbb{K}=(G,M,I)$ be a formal context and $n\in\mathbb{N}$ with
  $n<|M|$.  Find a subset $N\subseteq M$ with $|N|=n$ such that $r(N)\geq r(X)$
  for all $X\subseteq M$ where $|X|=n$. 
\end{problem}

Solving~\ref{rrp} is twofold infeasible. First, as $n$ increases does the number
of possible subset combinations. The determination of a maximal relevant set
requires the computation and comparison of $\binom{|M|}{|N|}$ different relative
relevances, which presents itself infeasible. Secondly, does the computation of
the relative relevance presume that the set of formal concepts is computed.
This states also an intractable problem for large formal contexts, which are the
focus for applications of the proposed relevance selection method.  To overcome
the first limitation we suggest an iterative approach. Instead of testings every
subset of size $n$ we construct $N\subseteq M$ by first considering all
singleton sets $\{m\}\subseteq M$. Consecutively, in every step $i$ where $X$ is
the so far constructed set we find $x\in M$ such that $r(X\cup\{x\})\geq
r(X\cup\{m\})$ for all $m\in M$.  This approach requires the computation of only
$\sum_{i=|M|-|n|+1}^{|M|} i$ different relative relevances and their
comparisons, which is simplified $n\cdot|M|-(n-1)\cdot n/2$. We call a set
obtained through this approach an \emph{iterative maximal relevant set} IMRS.
In fact the IMRS does not always correspond to the maximal relevant set. In
$(G,M,I)$ where $G=\{1,2,3,4\}$, $M=\{a,b,c,d\}$ and $I=\{(1,a),(1,c), (1,d),
(2,a), (2,b), (3,b), (3,c), (4,d)\}$ is $b$ the most relevant attribute, i.e.,
$r(b)>r(x)$ for all $x\in M\setminus\{b\}$. However, we find
$r(\{a,c\})>r(\{b,x\})$ for all $x\in M\setminus\{b\}$.  Hence, the relative
relevance of an IMRS indicates a lower bound for the relative relevance of the
maximal relevant set.

\subsection{Approximating RRP}
\label{sec:appr-relev}

Motivated by the computational infeasibility of~\ref{rrp} we investigate in this
section the possibility of approximating RRP, more specifically the
IMRS. Approaches for this approximation have to incorporate both aspects of the
relative relevance the structure of the concept lattice and the distribution of
the objects. Considering the former is not complicated due to~\cite[Proposition
30]{fca-book}, which states that for any $(G,M,I)$ is $\B{(G,N,I\cap(G\times
  N))}$ join preserving order embeddable into $\B{(G,M,I)}$ for any $N\subseteq
M$. Thus, this aspect can be represented through a quotient $|\B{\K}_{M\setminus
  N})|/|\B{\K}|$, which is a special case of the maximal common sub-graph
distance, see~\cite{BUNKE}.  Hence, whenever searching for the largest
$\B{(G,N,I\cap(G\times N))}$ the obvious choice is to optimize for large
contra-nominal scales in sub-contexts of $(G,M,I)$. For example, when selecting
three attributes in~\cref{runningexp} (left) the largest join preserving order
embeddable lattice would be generated by the set $\{b, c, d\}$. However, the
relative relevance of $\{b, c, g\}$ is significantly larger, in particular,
$r(\{b, c,d\})=17/33$ and $r(\{b, c, g\})=19/33$.

Considering the second requirement, the distribution of the objects on the
concept lattice, the sizes of the concept extents have to be incorporated. Since
they are unknown, unless we compute the concept lattice, we need a proxy for
estimating the influence of those. Accordingly, we want to reflect this with the
quotient $E(\K_{M\setminus N})/E(\K)$, which estimates the change of the object
distribution on the concept lattices when selecting a set $N\subseteq M$. This
quotient does employ a mapping $E:\mathcal{K}\to \mathbb{R}, \K\mapsto E(K)$,
which is to be found.  A natural candidate for this mapping would be information
entropy, as introduced by Shannon in~\cite{shannon}. He defined the entropy of a
discrete set of probabilities $p_1,\dotsc,p_n$ as $H=-\sum_{i\in I} p_i \log\
p_i$.  We adapt this formula to the realm of formal contexts as follows.

\begin{defn}
  \label{object-entropy}
  Let $\mathbb{K}=(G,M,I)$ be a formal context.
  Then the \textit{Shannon object information entropy of $\mathbb{K}$} is given as follows.
  \[E_{SE}(\mathbb{K})=  \sum_{g\in G} - \frac{|g''|}{|G|} ~log_2\left( \frac{|g''|}{|G|}\right)
  \]
\end{defn}

For this entropy function we employ the quotient $|g''|/|G|$, which does reflect
the extent sizes of the object concepts of $\K$. Obviously this choice does not
consider all concept extents. However, since every extent in a concept lattice
is either the extent of a object concept or the intersection of finitely many
extents of object concepts we see that Shannon object information entropy does
relate to all extents to some degree. We found another candidate for $E$ in the
literature~\cite{LOIA20181}. The authors there introduced an entropy function
which is roughly speaking the mean distance of the extents of object concepts to
the complete set objects.

\begin{defn} 
  \label{context-entropy}
  Let $\mathbb{K}=(G,M,I)$ be a formal context.
  Then the \textit{object information entropy of} $\mathbb{K}$ is given as follows.
  \[E_{OE}(\mathbb{K})= \frac{1}{|G|} \sum_{g\in G} \left( 1 - \frac{|g''|}{|G|}\right)
  \]
\end{defn}

We directly observe that this entropy decreases as the number of objects having
similar attribute sets increases. Furthermore, we recognize an essential
difference for $E_{OE}$ compared to $E_{SE}$. The Shannon object information
entropy reflects on the number of necessary bits to encode the formal
context. In contrary does the object information entropy reflect on the average
number of bits to encode an object from the formal context. To enhance the first
grasp of the just introduced functions as well as the relative relevance defined
in~\cref{def:relr} we want to investigate them on well known contextual
scales. In particular, the \emph{ordinal scale}
$\mathbb{O}_n\coloneqq([n],[n],\leq)$, the \emph{nominal scale}
$\mathbb{N}_n\coloneqq([n],[n],= )$, and the \emph{contranominal scale}
$\mathbb{C}_n\coloneqq([n],[n],\neq)$, where $[n]\coloneqq
\{1,\dotsc,n\}$. Since there is a bijection between the set $\{1,\dotsc,n\}$ to
the extent sizes $|g''|$ in an ordinal scale we obtain that
$E_{SE}(\mathbb{O}_n)=-\sum_{i=1}^{n} \frac{i}{n} log_2
\left(\frac{i}{n}\right)$ and
$E_{OE}(\mathbb{O}_n)=\frac{1}{n}\sum_{i=1}^n\left(1-\frac{i}{n}\right)=
\frac{1}{n}\frac{n(n+1)}{2n}=\frac{n+1}{2n}$. The former diverges to $\infty$
whereas the latter converges to $1/2$.  Based on the linear structure of
$\underline{\mathfrak{B}}(\mathbb{O}_n)$ we conclude that the set
$\mathfrak{B}(\K)\setminus \mathfrak{B}(\K)_{\{m\}}=\{(m',m'')\}$ for all $m\in
M$. So the relative relevance of the attribute $m\in M$ amounts to
$r(m)=1-(\sum_{i=1}^{n}i-|m''|)\ /\ \sum_{i=1}^{n}i=2|m''|/(n\cdot (n+1))$.
	
Both the nominal scale as well as the contranominal scale satisfy $g''=g$ for
all $g\in G$ for different reasons. We conclude that $E_{SE}$ and $E_{OE}$
evaluate respectively equally for $\mathbb{N}_{n}$ and $\mathbb{C}_{n}$. In
detail, $E_{SE}(\mathbb{N}_n)=E_{SE}(\mathbb{C}_n)=-\sum_{g\in G} \frac{1}{n}
log_2 \left(\frac{1}{n}\right)=\log_{2}(n)$ and
$E_{OE}(\mathbb{N})=E_{OE}(\mathbb{C})=\frac{1}{n}\sum_{g\in G} \left(
  1-\frac{1}{n}\right)=\frac{n-1}{n}$. For the relative relevance we observe
that $r(m)=r(n)$ for all $m,n\in M$ in the case of the
nominal/contranominal scale. This is due to the fact that every attribute is
part of the same number of concepts. For the nominal scale holds
$r(m)=1-\frac{2n-1}{2n}$ for all $m\in M$. Hence, as the number of attributes
increases does the relevance of a single attribute converge to zero.  The
relative relevance of an objects in the case of the contranominal scale is
$r(m)=1-\frac{\sum_{k=0}^{n} \binom{n}{k}(n-k)-\sum_{k=0}^{n-1}
  \binom{n-1}{k}(n-1-k)}{\sum_{k=0}^{n} \binom{n}{k}(n-k)}$ for all $m\in M$.

\begin{example} \label{exmpl:entropy} Revisiting our running
  example~\cref{runningexp} (right). This context has four objects with
  $\{B\}''=\{B,F,S\},\ \{F\}''=\{F\},\ \{D\}''=\{F,D\}\ \text{and}\
  \{S\}''=\{S\}$. Its entropies are given by $E_{OE}(\mathbb{K})= \frac{1}{4}
  \sum_{g\in G} \left( 1 - \frac{|g''|}{4}\right)\approx 0.56$ and
  $E_{SE}(\K)\approx 0.45$.
\end{example}

Considering both aspects discussed in this section we now want to introduce a
function which shall be capable of approximating RRP. 

\begin{defn} 
	\label{abstand}
	Let $\mathbb{K}=(G,M,I)$ and $\mathbb{K}_{\overline{N}}\coloneqq(G,N,I\cap
        (G\times N))$ be formal contexts with $N\subseteq M$. The
        \textit{entropic relevance approximation (ERA) of $N$} is defined as
	\[
          ERA(N)\coloneqq
          \frac{|\mathfrak{B}(\K_{\overline{N}})|}{|\mathfrak{B}(\K)|}
          \cdot \frac{E(\K_{\overline{N}})}{E(\K)}
	.\]
\end{defn}

First, the ERA compares the number of concepts in a given formal context to the
number of concepts in a sub-context on $N\subseteq M$. This reflects the
structural impact when restricting the attribute set. Secondly, an quotient is
evaluated where the entropy of $\K_{\overline{N}}$ is compared to the entropy of $\K$. When
using~\cref{abstand} for finding a subset $N\subseteq M$ with maximal (entropic)
relevance it suffices to compute $N$ such that $\B{\K_{\overline{N}}}\cdot E(\K_{\overline{N}})$ is
minimal. This task is essentially less complicated since we only have to compute
$\B{\K_{\overline{N}}}$ and $E(\K_{\overline{N}})$ for some comparable small formal context $\K_{\overline{N}}$.

\section{Experiments}
\label{sec:experiments}
\begin{figure}[t]
  \centering
  \includegraphics[width=0.48\textwidth]{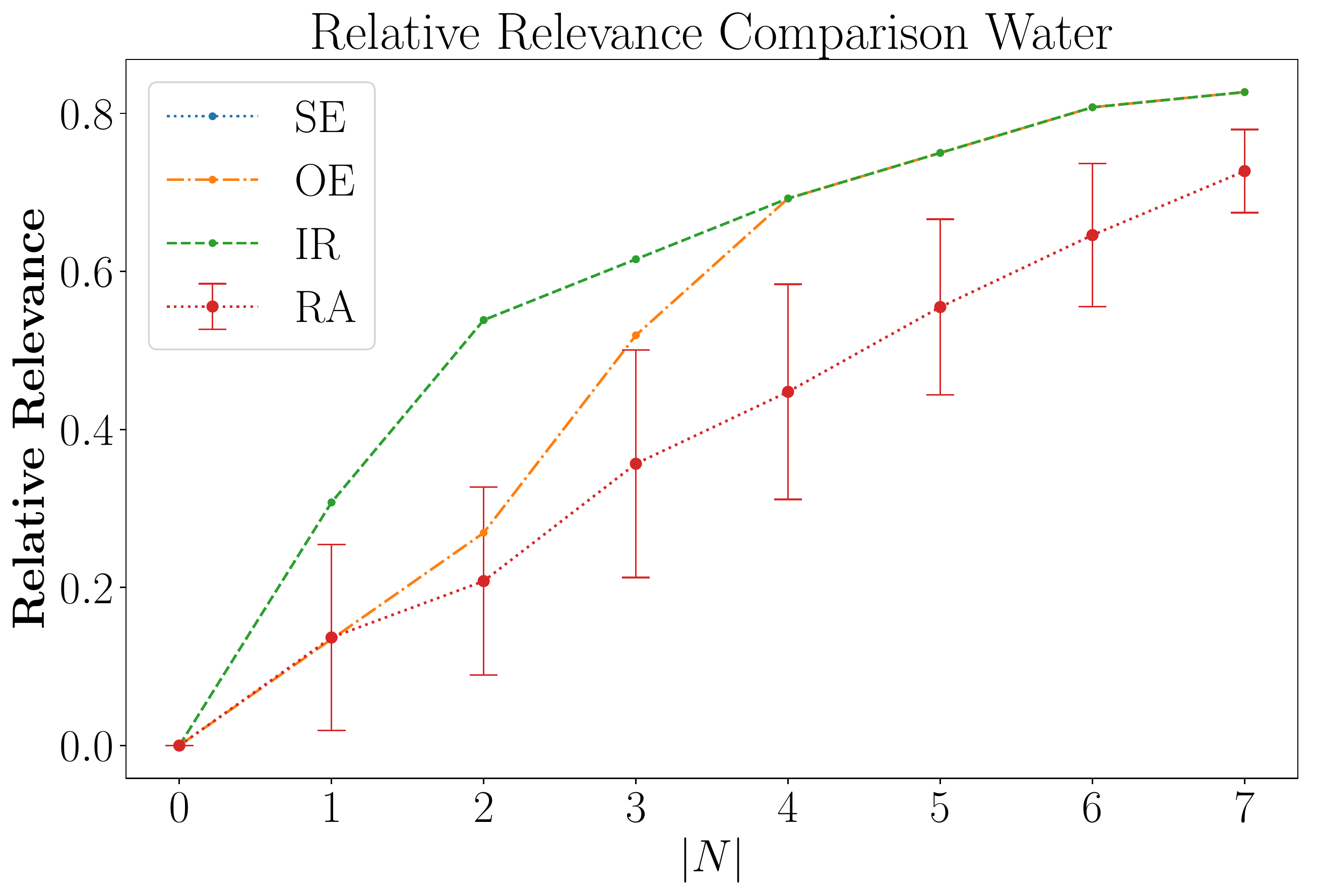}
  \includegraphics[width=0.48\textwidth]{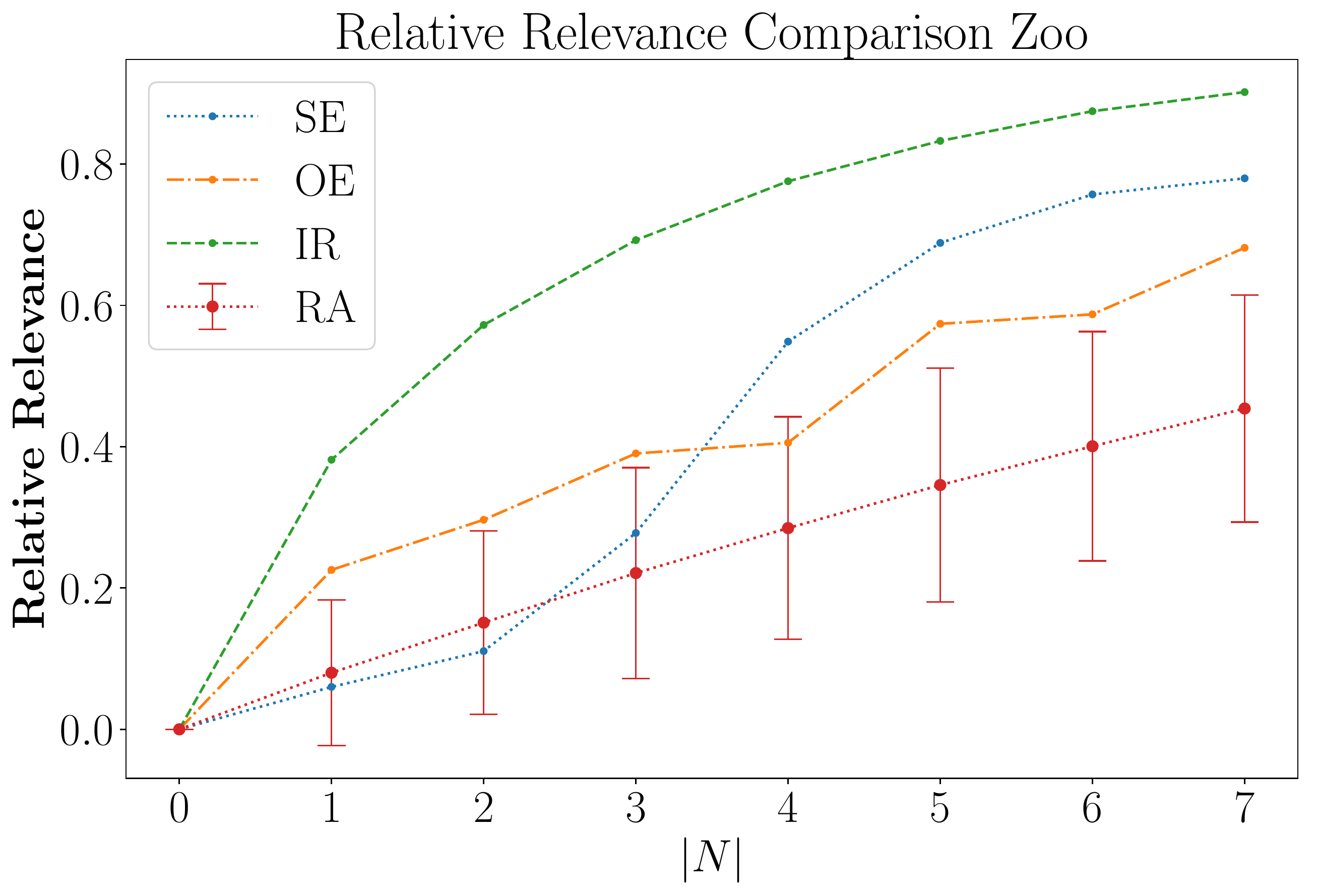}
  \caption{Relevance of attribute selections through entropy (SE,OE), IMRS (IR),
    and random selection (RA) for the ``Living beings in water'' (left)
    and the zoo context (right).}
  \label{fig:waterzoo}
\end{figure}

To assess the ability for approximating relative relevance
through~\cref{abstand} we carried out several experiments in the following
fashion. For all data set we computed the iterative maximal relevant subsets of
$M$ of sizes one to seven (or ten) in the obvious manner. We decided for those
fixed numbers for two reasons. First, using a relative number, e.g., 10\% of all
attributes, would still lead to an infeasible computation when the initial
formal context is very large. Secondly, formal contexts with up to ten
attributes permit a plenitude of research methods that are impracticable for
larger contexts, in particular, human evaluation.

Then we computed subsets of $M$ using ERA, for which we used both introduced
entropy functions, and their relative relevance. Finally, we sampled subsets of
$M$ randomly at least $|M|\cdot 10$ many times and computed their average
relative relevance as well as the standard deviation in relative relevance.

\subsection{Data Set Description}
\label{sec:data-set-description}
A total of 2678 formal contexts were considered in this experimental study. From
those were 2674 contexts excerpts from the BibSonomy
platform\footnote{\url{https://www.kde.cs.uni-kassel.de/wp-content/uploads/bibsonomy/}}
as described in~\cite{benz2010social}. All those contexts are equipped with an
attribute set of twelve elements and a varying number of objects. The particular
extraction method is described in detail in~\cite{BorchmannH16}. For the rest we
revisited three data sets well known in the realm of formal concept analysis,
i.e., \emph{mushroom, zoo, water}~\cite{Dua:2017, fca-book}, and additionally a
data set \emph{wiki44k} introduced in~\cite{rules-kg-learning}, which is based
on a 2014 Wikidata\footnote{\url{https://www.wikidata.org}} database dump. The
well-known \emph{mushroom} data set is a collection of 8124 mushrooms described
by 119 (scaled) attributes and exhibits 238710 formal concepts. The \emph{zoo}
data set possesses 101 animal descriptions using 43 (scaled) attributes and
exhibits 4579 formal concepts. The \emph{water} data set, more formally ``Living
beings and water'', has eight objects and nine attributes and exhibits 19 formal
concepts. Finally, the \emph{wiki44k} has 45021 objects and 101 attribute
exhibiting 21923 formal concepts.

\begin{figure}[t]
  \centering
  \includegraphics[width=0.48\textwidth]{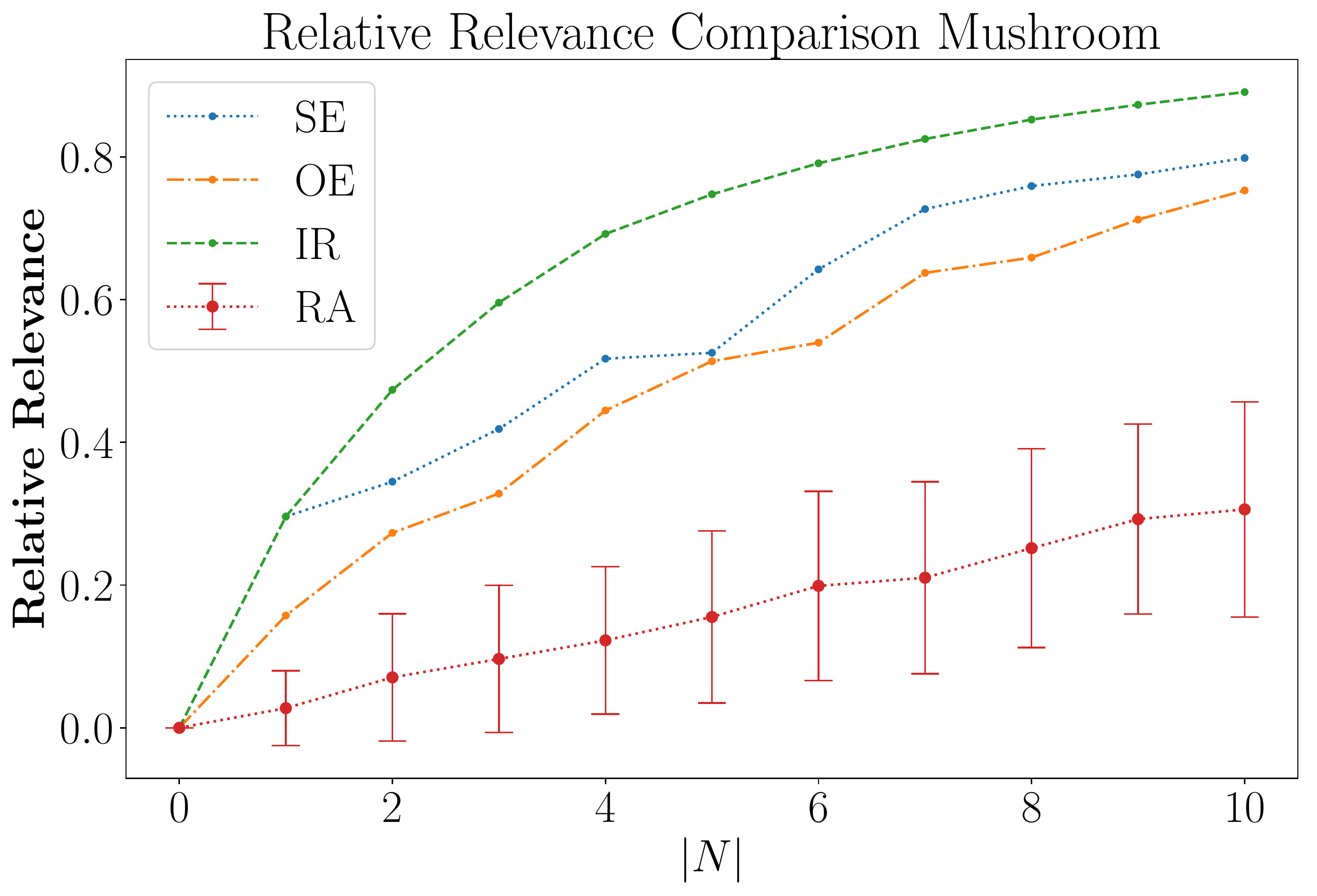}
  \includegraphics[width=0.48\textwidth]{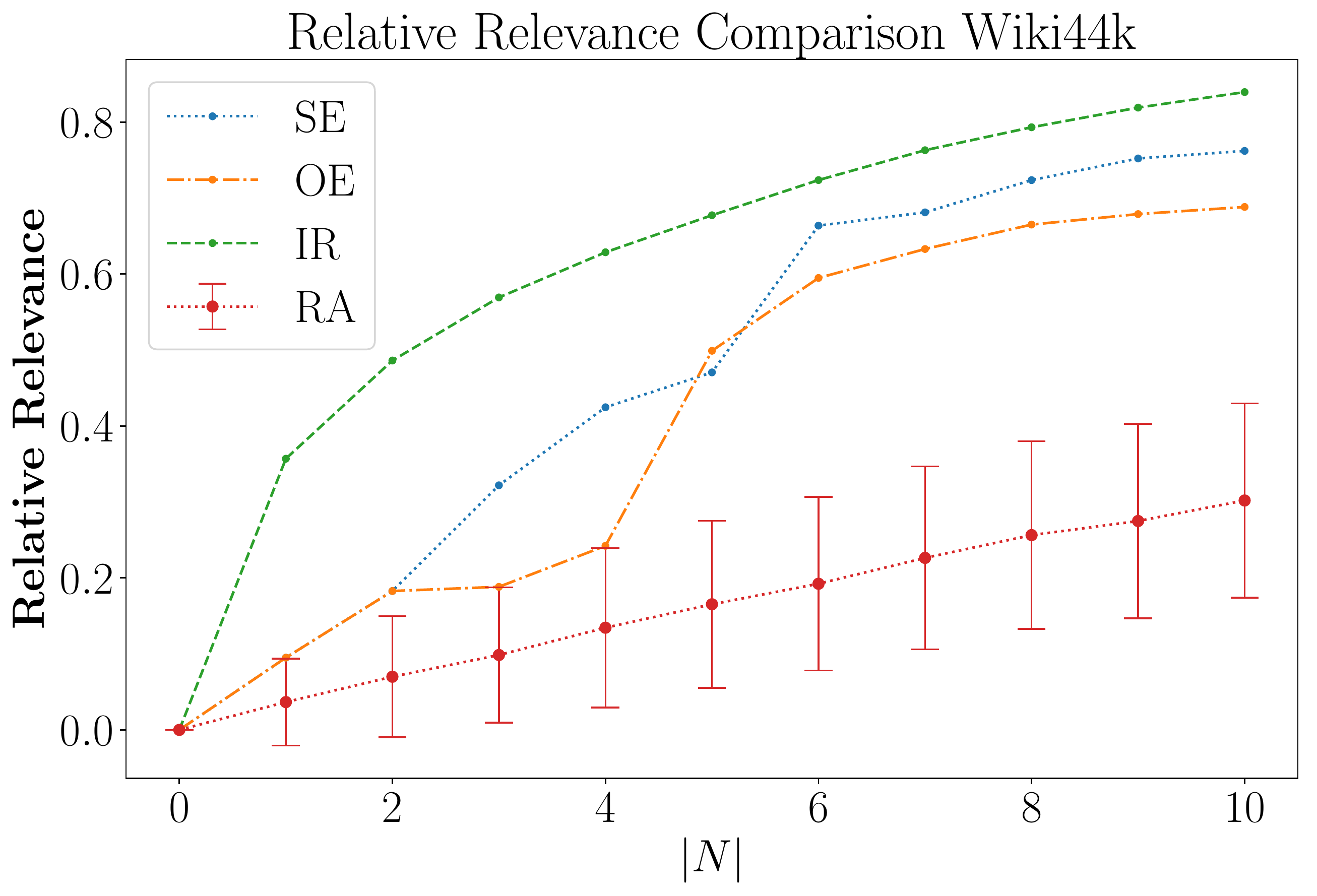}
  \caption{Relevance of attribute selections through entropy (SE, OE), IMRS
    (IR), and random selection (RA) for the mushroom (left) and the wiki44k context
    (right).}
  \label{fig:mushroomwiki}
\end{figure}

\begin{figure}[t]
  \centering
  \includegraphics[width=0.768\textwidth]{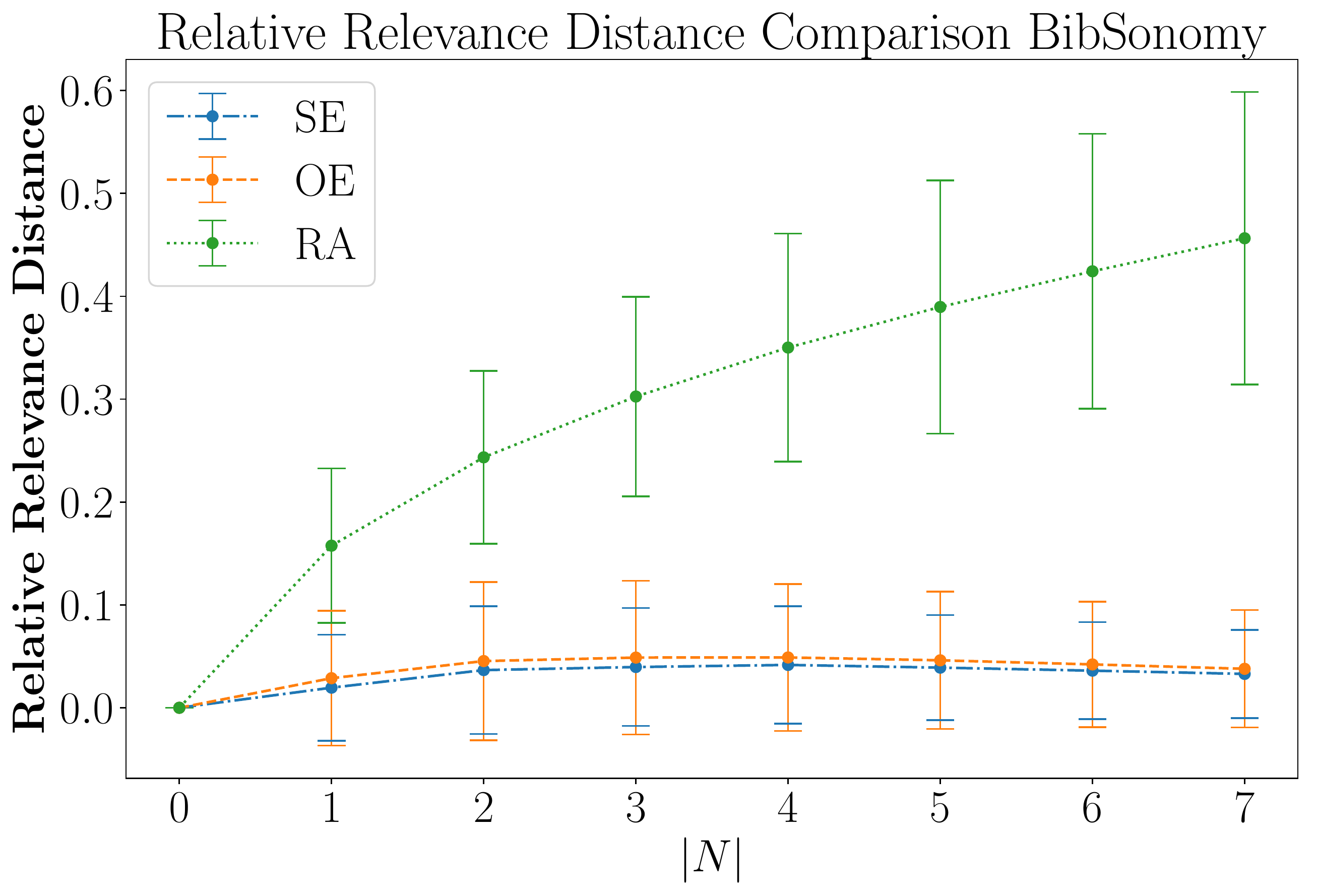}
  \caption{Average distance and standard deviation to IMRS for entropy and
    random based selections of $|N|$ attributes for 2674 formal contexts from
    BibSonomy.}
  \label{fig:bibsonomy}
\end{figure}

\subsection{Results}
\label{sec:results}

In~\cref{fig:mushroomwiki,fig:bibsonomy,fig:waterzoo} we depicted the results of
our computations. We observe in all experiments that the relative relevance of
the subsets found through the iterative approach are an upper bound for the
relative relevance of all subsets computed through entropic relevance
approximation or random selection, with respect to the same size of subset. In
particular we find IMRS of cardinality seven and above have a relative relevance
of at least 0.8. Moreover, the relative relevance of the attribute subsets
selected by both ERA versions (SE or OE) exceed the relative relevance of the
randomly selected subsets except for the Shannon object information entropy for
|N|=1 and |N|=2 in the zoo context. Principally we find for contexts containing
a small number of attributes (\cref{fig:waterzoo}) a large increase of the
distance between the relative relevance of the randomly selected attributes and
the attribute sets selected through the entropy approach. This characteristic
manifests in the relative relevance of both ERA selections excelling not only
the mean relative relevance of randomly chosen attribute sets but also the
standard deviation for subset sizes of $|N|=4$ and above. In the case of
contexts containing a huge number of attributes this observation can be made for
selections with $|N|=1$, already. Furthermore, the interval between the relative
relevance of the attribute subsets selected by both ERA versions and the
relative relevance of the randomly selected subsets is significantly larger than
in the case of contexts with small attribute set sizes. In general we may point
out that neither of the entropies seems preferable over the other in terms of
performance. In~\cref{fig:bibsonomy} we show the results for the experiment with
the 2674 formal contexts from BibSonomy. We plotted for all three methods,
ERA-OE/SE and random, the mean distance in relative relevance to the IMRS of the
same size together with the standard deviation. We detect a significant
difference for randomly chosen and ERA chosen sets with respect to their
relative relevance. The deviation for both ERA is bound by 0 and 0.12 . In
contrast, the relative relevance for randomly selected sets is bound by 0.09 and
0.6.

\subsection{Discussion}
\label{sec:discussion}
We found in our investigation that attribute sets obtained through the iterative
approach for relative relevance do have a high relevance value. Even though
their relative relevance is only an lower bound compared to the maximal relevant
set they do exhibit a relative relevance of 0.8 for attribute set sizes seven
and above. We conclude from this that iterative approach is a sufficient solution to
the relative relevance problem. Based on this we may deduct that entropic
relative approximation is also a good approximation for a solution to the RRP.
In particular, in large formal contexts investigated in this work the
approximation was even better than in the smaller ones.

\section{Conclusion}
\label{sec:conclusion}

By defining the relative relevance of attribute sets in formal contexts we
introduced a novel notion for attribute selection. This notion respects both the
structure of the concept lattice and the distribution of the objects on it. To
overcome computational limitations, which arised from the notion of relative
relevance, we introduced an approximation based on two different entropy
functions adapted to formal contexts. For this we used a combination of two
factors. The change in the number of concepts and the change in entropy that
arise by the selection of an attribute subset. The experimental evaluation for
relative relevance as well as the entropic approximation seem to comply with the
theoretical modeling. 

We may conclude our work with two open questions. First, even though IMRS seems
a good choice for relevant attributes we suspect that computing the maximal
relevant set, with respect to RRP, can be achieved more feasible as presented in
this work. Secondly, so far our justification for RRP is based on theoretical
assumptions and a basic experimental study. We imagine, and are curious, if
maximal relevant attribute sets are also employable in supervised machine
learning setups. For example, one may perceive the task of adding a new object
to a given formal context as instance of such a setup. The question is, how
capable is the context to add this object to an already existing concept.

\sloppy
\printbibliography

\end{document}